\documentclass[fleqn,11pt]{article}
\usepackage{hyperref}
\usepackage{amsmath, amssymb, amsfonts, amsthm}
\usepackage{multirow}
\usepackage{graphicx}
\usepackage{footnote}
\usepackage{caption}
\usepackage{color}
\usepackage{algorithm}
\usepackage{algpseudocode}
\usepackage{booktabs}
\usepackage{adjustbox}
\usepackage{array}
\usepackage{subfigure}
\usepackage{mathabx} 
\usepackage{ragged2e} 
\usepackage{bm} 
\usepackage{titlesec}
\usepackage{float, epsfig}
\usepackage{latexsym}
\usepackage{url}

	\newtheorem{definition}{Definition}

	\newtheorem{theorem}{Theorem}[section]
	\makeatletter
	\newcommand{\Rmnum}[1]{\expandafter\@slowromancap\romannumeral #1@}
	\makeatother
	
\begin{document}\sloppy
		
\title{Fusion Sampling Validation in Data Partitioning for Machine Learning}
	
\author{\footnote{${}^{1}$ Corresponding author: Tel.: +27 962 8135. Email address: caston.sigauke@univen.ac.za  (Caston Sigauke)}\bf ${}^{1}$Christopher Godwin Udomboso, ${}^{2}$Caston Sigauke,  and ${}^{1}$Ini Adinya\\ \\
${}^{1}$Department of Statistics, University of Ibadan, Nigeria \\
${}^{2}$Department of Mathematical and Computational Sciences \\ University of Venda, South Africa \\
${}^{1}$Department of Statistics, University of Ibadan, Nigeria}
\maketitle
			
\begin{abstract}
\noindent 	Effective data partitioning is known to be crucial in machine learning. Traditional cross-validation methods like K-Fold Cross-Validation (KFCV) enhance model robustness but often compromise generalisation assessment due to high computational demands and extensive data shuffling. To address these issues, the integration of the Simple Random Sampling (SRS), which, despite providing representative samples, can result in non-representative sets with imbalanced data. The study introduces a hybrid model, Fusion Sampling Validation (FSV), combining SRS and KFCV to optimise data partitioning. FSV aims to minimise biases and merge the simplicity of SRS with the accuracy of KFCV. The study used three datasets of 10,000, 50,000, and 100,000 samples, generated with a normal distribution (mean 0, variance 1) and initialised with seed 42. KFCV was performed with five folds and ten repetitions, incorporating a scaling factor to ensure robust performance estimation and generalisation capability. FSV integrated a weighted factor to enhance performance and generalisation further. Evaluations focused on mean estimates (ME), variance estimates (VE), mean squared error (MSE), bias, the rate of convergence for mean estimates (ROC\_ME), and the rate of convergence for variance estimates (ROC\_VE). Results indicated that FSV consistently outperformed SRS and KFCV, with ME values of 0.000863, VE of 0.949644, MSE of 0.952127, bias of 0.016288, ROC\_ME of 0.005199, and ROC\_VE of 0.007137. FSV demonstrated superior accuracy and reliability in data partitioning, particularly in resource-constrained environments and extensive datasets, providing practical solutions for effective machine learning implementations. \\ 
	
\noindent \textbf{Keywords:} Data partitioning, Cross-validation, Hybridisation, Machine learning, Sampling.
\end{abstract}
%

\section{Introduction} \label{sec:1.0}

\subsection{Overview} \label{sec:1.1}
Machine learning for high-cost systems has evolved since McCulloch and Pitts's 1940s work \cite{mcculloch1943}. Early Shallow Neural Networks (ShNN) transformed computer vision and natural language processing \cite{cybenko1989, hornik1989}, enabling applications in autonomous vehicles and medical diagnosis \cite{chen2017,krizhesky2017, redmon2016}. Scalability is crucial as tasks grow complex, especially in resource-constrained environments. Balancing bias and variance is vital in supervised learning for accurate classifiers \cite{hastie2005}. Effective data partitioning enhances model performance \cite{haynes2006, larranaga2006, pereira2009}. While cross-validation boosts robustness, nested cross-validation mitigates generalisation issues \cite{kohavi1995, varma2006}.

\subsection{Literature review} \label{sec:1.2}
Managing model complexity involves techniques like pruning, regularization \cite{reed1993, lecun1990, bishop1995}, early stopping, and ensemble averaging \cite{skouras1994, tetko1995}. Hyperparameter optimisation requires meticulous data partitioning \cite{lorraine2020}. Scalable architectures like Hadoop enhance big data analytics \cite{chen2014, lazar2018}, despite challenges \cite{genuer2017}. Advances include Deep Neural Networks (DNNs) for robust data representation \cite{huang2017} and lightweight designs \cite{han2016, courbariaux2015, howard2017}. Techniques like robustness verification \cite{downing2023} and scalable subsampling \cite{wu2024} boost efficiency. SCAN customises models for edge devices \cite{zhang2019}, and integrating supervised and unsupervised learning expands capabilities \cite{tetko1997, korjus2016}. Efficient hyperparameter optimisation enhances computational efficiency \cite{mlodozeniec2023}.

Studies underscore the efficacy of K-fold cross-validation in enhancing model accuracy across various domains \cite{kohavi1995, bishop2006, varma2006, bengio2004}. Previous research has shown that integrating simple random sampling with K-fold cross-validation can improve model evaluation by reducing biases introduced by data ordering \cite{kohavi1995, varma2006, bengio2004}. Simple random sampling ensures that training and validation samples are representative, enhancing model generalisation. Despite its benefits, random sampling can lead to non-representative training sets, particularly with imbalanced data distributions. K-fold cross-validation, while robust, demands significant computational resources for large datasets and requires meticulous data shuffling. To mitigate these challenges, we develop in this paper a hybrid approach which combines both techniques to optimise data partitioning. 

\subsection{Contribution and research highlights} \label{sec:1.3}
Based on the literature review in Section \ref{sec:1.2}, the present study proposes Fusion Sampling Validation (FSV). This new hybrid data partitioning technique seeks to leverage the merits of both Simple Random Sampling (SRS) and K-fold cross-validation (KFCV) while eliminating their drawbacks at the same time. FSV assists in stabilising the model, minimising bias and variance, and optimising computational efficiency, making it a superior option for large data.

Some of the research highlights are:
\begin{itemize}
	\item 	FSV consistently shows the highest stability, lowest variance, lowest bias, and lowest Mean Squared Error (MSE) across a range of sample sizes, i.e., from $N = 10,000 to N = 100,000$.
	\item FSV has improved convergence rates for mean and variance estimates relative to SRS and KFCV.
	\item In contrast to KFCV, which is computationally demanding on large datasets, FSV enhances computational efficiency without sacrificing performance.
	\item SRS's non-representative sampling issues are evaded in FSV for more effective generalisation in imbalanced data.
	\item While KFCV still works well with smaller datasets, FSV is the best approach for larger datasets, sacrificing efficiency and accuracy.
\end{itemize}

The rest of the paper is organised as follows. Section \ref{sec:2.0} discusses in detail the methods used in this study. The results are presented in Section \ref{sec:3.0}, with a presentation of a detailed discussion of the results in Section \ref{sec:4.0}. The conclusion is given in Section \ref{sec:5.0}.

\section{Methodology} \label{sec:2.0}
The flowchart of machine learning data partitioning techniques is given in Figure \ref{fig:flowchart}.

\vspace{-1.7in}
\begin{figure}[H]
	\noindent\makebox[\linewidth]{%
		\includegraphics[width=1.2\textwidth]{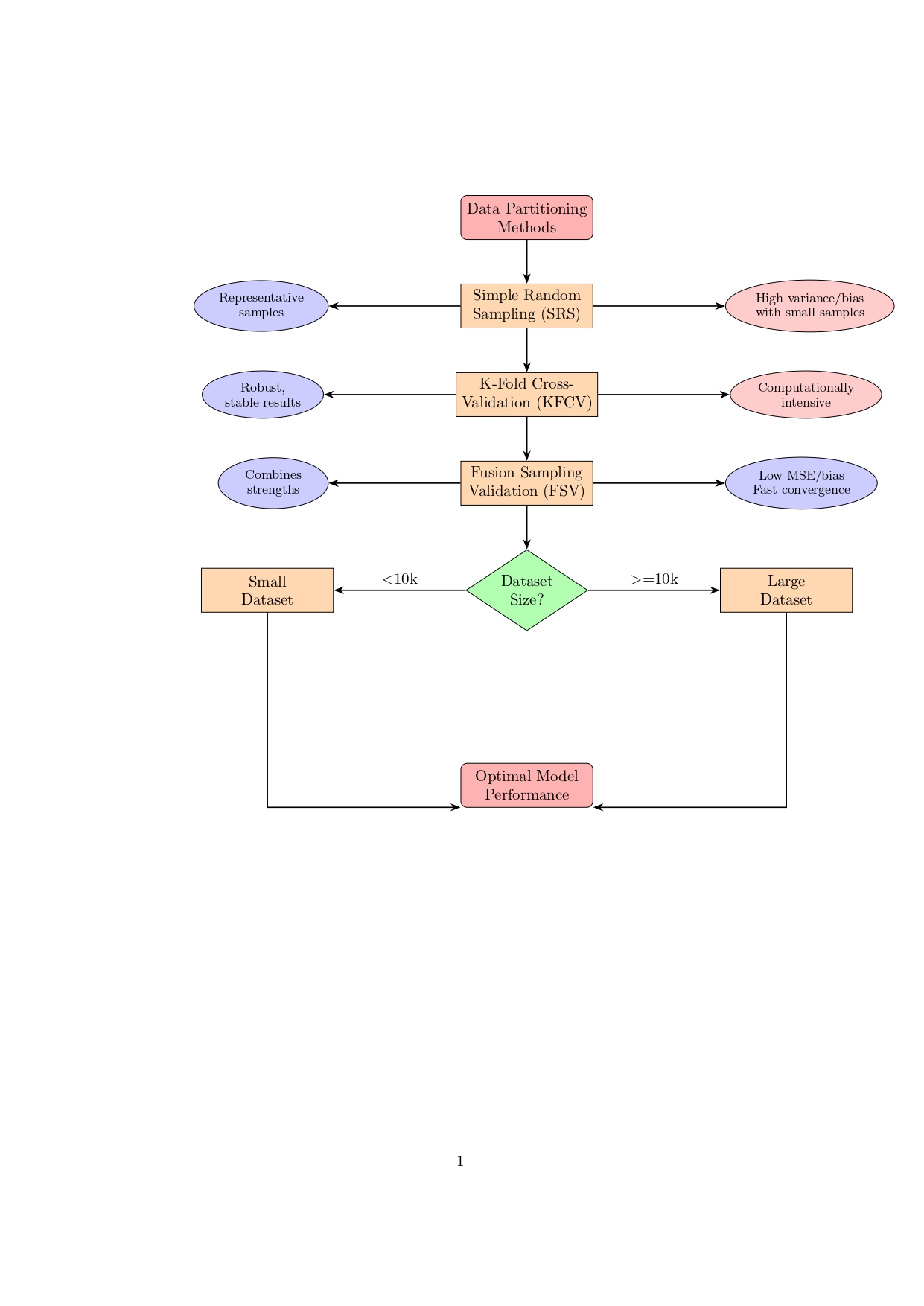}
	}
	\vspace{-4.2in}
	\caption{Flowchart of machine learning data partitioning techniques.}
	\label{fig:flowchart}
\end{figure}

\subsection{Problem formulation and notations} \label{sec:2.1}
Let $D=\{(x_i,y_i)\}_{i=1}^n$ be a dataset with $n$ data points, where $x_i$ are the inputs, and $y_i$ are the corresponding labels.
	
\textbf{Notations:}
\begin{tabular}{ll}
$S\subseteq D$ & A sample selected through simple random sample SRS($m$) \\
$K$ & Number of folds in k-fold cross-validation KFCV \\
$S_i$ & The $i^{th}$ fold of the sample $S$ \\
$M$ & The neural network model \\
$L(M,S)$ & The loss function for the model $M$ on data $S$ \\
$\hat{L}(M,S)$ & The empirical loss of the model \\
$|S|$  & cardinality of $S$
\end{tabular}
	
\subsection{Simple Random Sampling}
\textbf{Definition 1:} Simple Random Sampling (SRS) is the selection of a subset $S$ from $D$ such that each element of $D$ has an equal probability of being included in $S$.
	
	For each $x_i\in D$, let $I_i$ be an indicator random variable where
	\[
	I_i = \begin{cases}
		1 & \text{if } x_i \in S \\
		0 & \text{otherwise}
	\end{cases} \tag{1}
	\]
	The probability $P(S_i=1) = \frac{m}{n}$, where $m = |S|$ and $n = |D|$. 
	
	\textbf{Properties:} Now the expected value and variance of $S$ are given as follows:
	\begin{align}
		E(S) &= \sum_{i=1}^n P(S_i=1) = m \tag{3} \\
		Var(S) &= \sum_{i=1}^n Var(S_i) = n \cdot \frac{m}{n} \left(1 - \frac{m}{n}\right) \tag{4} \\
		&= m\left(1 - \frac{m}{n}\right) \tag{5}
	\end{align}
	
	\textbf{Theorem 1 (Consistency of the SRS):} As the sample size $m$ approaches the dataset size $n$, the sample $S$ obtained by the SRS becomes representative of the dataset $D$.
	
	\textbf{Proof:} Consider a dataset $D$ with $n$ data points $\{x_i\}_{i=1}^n$. SRS selects a subset $S\subseteq D$ of size $m$. Each data point $x_i\in D$ has an equal chance of being included in $S$, specifically
	\begin{equation}
	P(x_i \in S) = \frac{m}{n} 
	\end{equation}
	As $m$, the sample size, gets closer to $n$, the dataset size, the proportion $|S|/n$ converges to $m/n$. This convergence is assured by the law of large numbers, which states that the sample mean approaches the population mean as sample size increases.
	
	Therefore when $m$ is sufficiently large compared to $n$, $S$ becomes representative of $D$, meaning that the statistical characteristics and properties of $D$ are present in $S$.

\subsection{K-fold cross-validation}
		
		\begin{definition}
			The K-Fold Cross-Validation (KFCV) splits the dataset $S$ into $k$ folds. For each fold $S_i$, the model is trained on $S \setminus S_i$ and validated on $S_i$. The scenario is as follows:
		\end{definition}
		
		\begin{tabular}{|c|c|c|c|c|c|c|} 		\hline
			1     & 2     & 3     & 4          & 5     & \dots & K  \\ \hline           
			train & train & train & validation & train & \dots & train \\
			\hline
		\end{tabular} \\ 
		
	\vspace{0.2in}
		For the $k^{\mbox{th}}$ component, fourth on the above table, we fit a model to the $K-1$ components of the data. We then calculate the prediction error of the fitted model in predicting the $k^{\mbox{th}}$ component of the data.
	
		Let $S = \bigcup_{i=1}^k S_i$ \hfill 
		
		with $S_i \cap S_j = \emptyset$ for $i \neq j$ and $|S_i| \approx m/k$.
		
		For each fold $i$:
		\begin{align*}
			S_{\text{train}}^{(i)} &= S \setminus S_i \quad  \\
			S_{\text{val}}^{(i)} &= S_i \quad 
		\end{align*}
		
		\textbf{Expected loss:}
		The expected loss is given as
		\begin{equation}
			L_{k_{\text{fold}}}(M,S) = E\left[L(M,S_i) \right] 
		\end{equation}
		
		which can be estimated empirically using the following equation:
		\begin{equation}
			\hat{L}_{k_{\text{fold}}}(M,S) = \frac{1}{k} \sum_{i=1}^k L(M,S_i) \quad \text{(10)}
		\end{equation}
		
		\begin{theorem}[Properties of KFCV]
			KFCV provides an unbiased estimate of the model's performance on unseen data.
		\end{theorem}
		
		\begin{proof}
			Consider a dataset $S$ partitioned into $k$ subsets or folds $\{S_1,S_2,\ldots,S_k\}$. For each fold $S_i$,
			\begin{itemize}
				\item train the machine learning model $m$ on the remaining $S \setminus S_i$
				\item evaluate the model on $S_i$ to compute the loss $L(M,S_i)$
			\end{itemize}
			
			The average validation loss across all folds, known as the KFCV error $\hat{L}_{k\text{-fold}}(M,S)$ is,
			\begin{equation}
				\hat{L}_{k\text{-fold}}(M,S) = \frac{1}{k} \sum_{i=1}^k L(M,S_i) \quad \text{(11)}
			\end{equation}
			
			By the law of large numbers, as $k$, the number of folds, approaches $n$, the dataset size, $\hat{L}_{k\text{-fold}}(M,S)$ converges to $L(M,D)$, which is the expected loss of the model on unseen data $D$.
			
			This convergence ensures that KFCV provides an unbiased estimate of how well the model will perform on new, unseen data.
			
			A scale factor $\lambda$ is introduced to ensure reliable estimates of model's performance and generalization ability.
		\end{proof}
		
		\begin{theorem}
			Let $\lambda_i$ be a scale factor for each fold $S_i, i = 1,...,k$. The modified expected loss is
			\begin{equation}
				\hat{L}_{k_{\text{fold}}}^\lambda (M,S) = \frac{1}{k} \sum_{i=1}^k \lambda_i L(M,S_i).
			\end{equation}
			
			Properties to Prove:
			\begin{enumerate}
				\item Unbiasedness: $\mathbb{E}[\hat{L}_{k_{\text{fold}}}^\lambda (M,S)] = L(M,D)$, where $D$ is unseen data.
				\item Convergence: $\hat{L}_{k_{\text{fold}}}^\lambda (M,S) \to L(M,D)$ as $k \to n$, the dataset size $n$.
				\item Variance Comparison between $\hat{L}_{k_{\text{fold}}}^\lambda (M,S)$ and $\hat{L}_{k\text{-fold}}(M,S)$.
			\end{enumerate}
		\end{theorem}
		
		\begin{proof}
			1. \textbf{Unbiasedness}
			
			We need to show that $\mathbb{E}[\hat{L}_{k_{\text{fold}}}^\lambda (M,S)] = L(M,D)$, where $L(M,D)$ is the expected loss on unseen data $D$.
			
			\begin{align*}
				\mathbb{E}[\hat{L}_{k_{\text{fold}}}^\lambda (M,S)] &= \mathbb{E}\left[\frac{1}{k} \sum_{i=1}^k \lambda_i L(M,S_i)\right] \\
				&= \frac{1}{k} \sum_{i=1}^k \lambda_i \mathbb{E}[L(M,S_i)] \\
			\end{align*}
			
			Since $\mathbb{E}[L(M,S_i)] = L(M,D)$, assuming the model's performance is consistent across folds and unbiased on $D$;
			
			\begin{align*}
				\mathbb{E}[\hat{L}_{k_{\text{fold}}}^\lambda (M,S)] &= \frac{1}{k} \sum_{i=1}^k \lambda_i L(M,D) \\
				&= L(M,D) \cdot \frac{1}{k} \sum_{i=1}^k \lambda_i
			\end{align*}
			
			For unbiasedness, $\sum_{i=1}^k \lambda_i = k$. Therefore,
			
			\begin{align*}
				\mathbb{E}[\hat{L}_{k_{\text{fold}}}^\lambda (M,S)] &= L(M,D)
			\end{align*}
			
			Thus, $\hat{L}_{k_{\text{fold}}}^\lambda (M,S)$ is an unbiased estimator of $L(M,D)$.
			
			2. \textbf{Convergence}
			
			To show convergence, consider the behaviour of $\hat{L}_{k_{\text{fold}}}^\lambda (M,S)$ as $k \to n$.
			
			As $k$ approaches $n$, each $S_i$ approaches a single sample in size, and $\hat{L}_{k_{\text{fold}}}^\lambda (M,S)$ approaches $L(M,D)$.
			
			This follows from the law of large numbers.
			
			3. \textbf{Variance Comparison}
			
			In comparing variances, we note that
			
			\begin{align*}
			\text{Var}(\hat{L}_{k_{\text{fold}}}^\lambda (M,S)) &= \text{Var}\left(\frac{1}{k} \sum_{i=1}^k \lambda_i L(M,S_i)\right) \\
				&= \frac{1}{k^2} \sum_{i=1}^k \lambda_i^2 \text{Var}(L(M,S_i)),
			\end{align*}
			assuming that $L(M,S_i)$ are independently and identically distributed.
			Compare this with the variance of $\hat{L}_{k_{\text{fold}}}(M,S)$,
			\begin{align*}
				\text{Var}(\hat{L}_{k_{\text{fold}}}(M,S)) &= \frac{1}{k^2} \sum_{i=1}^k \text{Var}(L(M,S_i))
			\end{align*}
			The choice of $\lambda_i$ can potentially reduce $\text{Var}(\hat{L}_{k_{\text{fold}}}^\lambda (M,S))$ in comparison to $\text{Var}(\hat{L}_{k_{\text{fold}}}(M,S))$, depending on the selection procedure for $\lambda_i$.
		\end{proof}
		
\subsection{The hybrid model: Fusion sampling validation}

The hybrid model, the Fusion Sampling Validation (FSV), combines the characteristics of the SRS and KFCV through iterative compounding. This involves multiple iterations of random sampling followed by KFCV, averaging the results to achieve a more stable performance measure.

The weighted factor $\alpha$ is introduced into FSV, since it serves as a scaling factor that adjusts the contribution of each iteration's performance measure to the final compounded metric $L^*$. It controls over generalisation and balances bias-variance by scaling iteration contributions, reducing sensitivity to sample characteristics and improving the model's robustness and reliability. The choice of $\alpha$ is typically between 0.8 and 1.0, based on empirical testing 
		(\cite{breiman1996},\cite{freund1997},\cite{kohavi1995},\cite{tibshirani1996},\cite{bergstra2012}).
		
\begin{theorem}[Mathematical foundations of the weighted factor $\alpha$]
Let $L_t$ be the performance measure for iteration $t$ of the hybrid model. If $\alpha$ is a weighted factor, then the compounded performance measure $L^*$ is given by
\begin{equation}
L^* = \frac{1}{T} \sum_{t=1}^T \alpha L_t
\end{equation}
\end{theorem}
		
\begin{proof}
Let $L_t$ be the performance measure obtained from iteration $t$ of the hybrid model, where $t=1,2,\ldots,T$, and is given by
\begin{equation}
L^* = \frac{1}{T} \sum_{t=1}^T L_t
\end{equation}
		
Let $\alpha$ be a weighted factor that scales each performance measure $L_t$, such that
\begin{equation}
L_t' = \alpha L_t,
\end{equation}
where $L_t'$ is the weighted performance measure for iteration $t$.
			
Let the updated performance be given as
\begin{equation}
L^* = \frac{1}{T} \sum_{t=1}^T L_t'
\end{equation}
and substituting for $L_t'$, $L^*$ becomes
\begin{equation}
L^* = \frac{1}{T} \sum_{t=1}^T \alpha L_t
\end{equation}
			
Since $\alpha$ is a constant, it is factored out so that
\begin{equation}
L^* = \alpha \left( \frac{1}{T} \sum_{t=1}^T L_t \right)
\end{equation}
		
Therefore, $L^*$ becomes
\begin{equation}
L^* = \alpha L^{\text{orig}},
\end{equation}
where $L^{\text{orig}}$ is the original compounded performance measure.
\end{proof}
		
\begin{algorithm}
\caption{The FSV method.}
\begin{algorithmic}[1]
		\State Initialize $L^* \gets 0$ (cumulative compounded performance measure)
				\State Choose a weighted factor $\alpha$
				\For{each iteration $t$ from 1 to $T$}
				\State Select a random sample $S_t$ from $D$ of size $n$:
				\State $S_t = \{X_i \in D \mid i \in I_t \}$, where $I_t \subseteq \{1,2,\ldots,N\}$ and $|I_t| = n$
				\State Partition $S_t$ into $k$ folds: $S_{t,1}, S_{t,2}, \ldots, S_{t,k}$, each containing $n/k$ instances
				\For{each fold $S_{t,i}$}
				\State Train the model $k$ times
				\State Use $k-1$ folds for training: $S_t \setminus S_{t,i} = \bigcup_{j \neq i} S_{t,j}$
				\State Validate the model on fold $S_{t,i}$
				\State Compute the performance measure $L(\theta_{t,i}, S_{t,i})$,
				\State where $\theta_{t,i}$ is the model performance trained on $S_t \setminus S_{t,i}$
				\EndFor
				\State Compute the average performance measure $L_t$ for iteration $t$:
				\State $L_t = \frac{1}{k} \sum_{i=1}^k L(\theta_{t,i}, S_{t,i})$
				\State Update compounded performance measure $L^*$:
				\State $L^* = L^* + \alpha L_t$
				\EndFor
				\State Average the compounded performance measure estimated over $T$ iterations:
				\State $L^* = L^*/T$
	\end{algorithmic}
\end{algorithm}
	
	\newpage	
\subsubsection{Unbiasedness of the FSV method}
		
The hybrid model provides an unbiased estimate of the true performance,
\begin{equation}
\mathbb{E}(\hat{L}) = \mathbb{E}\left( \frac{1}{T} \sum_{t=1}^T L_t \right) = \frac{1}{T} \sum_{t=1}^T \mathbb{E}(L_t)
\end{equation}
		
Since each $L_t$ is based on KFCV of a $S_t$, we then write
\begin{equation}
\mathbb{E}[L(L_t)] = \frac{1}{k} \sum_{i=1}^k \mathbb{E}[L(\theta_{t,i}; S_{t,i})]
\end{equation}
		
Given that each $S_{t,i}$ is representative of $S_t$,
\begin{equation}
\mathbb{E}[L(\theta_{t,i}; S_{t,i})] = L(\theta; S_t)
\end{equation}
		
Thus,
\begin{equation}
\mathbb{E}(L_t) = L(\theta; S_t)
\end{equation}
		
Since each $S_t$ is a SRS of $D$, then
\begin{equation}
\mathbb{E}[L(\theta; S_t)] = L(\theta; D)
\end{equation}
		
Hence,
\begin{equation}
\mathbb{E}(\hat{L}) = L(\theta; D)
\end{equation}
		
\subsubsection{Variance}
The variance of the hybrid model $\hat{L}$ can be decomposed into two parts:
\begin{enumerate}
\item Variance due to SRS:
\begin{equation}
\text{Var}[L(\theta; S_t)] = \frac{\sigma^2}{n} \left(1 - \frac{n}{N}\right)
\end{equation}
			
\item Variance due to KFCV: Let $\text{Var}_k (L)$ denote the variance due to KFCV in the sample $S_t$, then
\begin{equation}
\text{Var}_k (L) = \frac{1}{k} \sum_{i=1}^k \text{Var}[L(\theta_{t,i}; S_{t,i})]
\end{equation}
\end{enumerate}
		
		Combining these variances for each iteration, the total variance of the hybrid model is
		\begin{equation}
			\text{Var}(\hat{L}) = \frac{1}{T} \sum_{t=1}^T \left[ \frac{\sigma^2}{n} \left(1 - \frac{n}{N}\right) + \frac{1}{k} \sum_{i=1}^k \text{Var}[L(\theta_{t,i}; S_{t,i})] \right]
		\end{equation}
		
		Assuming the variance within each fold $S_{t,i}$ is similar across iterations, then
		\begin{equation}
			\text{Var}(\hat{L}) = \frac{\sigma^2}{n} \left(1 - \frac{n}{N}\right) + \frac{1}{k} \sum_{i=1}^k \text{Var}[L(\theta_i; S_i)]
		\end{equation}
		
		Since the hybrid model averages over $T$ iterations, the variance reduces by a factor of $T$:
		\begin{equation}
			\text{Var}(\hat{L}) = \frac{1}{T} \left[ \frac{\sigma^2}{n} \left(1 - \frac{n}{N}\right) + \frac{1}{k} \sum_{i=1}^k \text{Var}[L(\theta_i; S_i)] \right]
		\end{equation}
		
\subsubsection{Error bounds}
		
		To derive the error bounds for the hybrid model, we utilize concentration inequalities such as the Chebyshev and Hoeffding's inequalities.
		
		Let $\hat{L}$ be the performance measure with mean $L(\theta;D)$, and variance $\text{Var}(\hat{L})$. Using the Chebyshev inequality,
		\begin{equation}
			P(|\hat{L} - L(\theta;D)| \geq k \sqrt{\text{Var}(\hat{L})}) \leq \frac{1}{k^2}
		\end{equation}
		
		Substituting the variance, and letting
		\begin{equation}
			\sigma_{\text{hyb}}^2 = \frac{\sigma^2}{n} \left(1 - \frac{n}{N}\right) + \frac{1}{k} \sum_{i=1}^k \text{Var}[L(\theta_{t,i}; S_{t,i})]
		\end{equation}
		
		Then
	\begin{equation}
		P\left(|\hat{L} - L(\theta;D)| \geq k \sqrt{\frac{\sigma_{\text{hyb}}^2}{T}}\right) \leq \frac{1}{k^2}
	\end{equation}
			
			This gives the error bound for error of the hybrid model.
			
			If the performance measures $L_t$ are bounded, we can use Hoeffding's inequality for a tighter bound. Suppose $L_t \in [a,b]$, then
			\begin{equation}
				P[|\hat{L} - L(\theta;D)| > \epsilon] \leq 2e^{-\frac{2T\epsilon^2}{(b-a)^2}}
			\end{equation}
			
			\begin{theorem}[Convergence of the hybrid performance measure]
				By the law of large numbers, as $T \to \infty$, the performance measure $\hat{L}$ converges almost surely to the expected value $L(\theta;D)$.
			\end{theorem}
			
			\begin{proof}
				Let $L_t$ be the performance measure for each iteration. The compounded performance measure is the average of $L_t$,
				\begin{equation}
					\hat{L}_t = \frac{1}{T} \sum_{t=1}^T L_t
				\end{equation}
				
	By the strong law of large numbers, the sample average converges almost surely to the expected value $\mathbb{E}(L_t) = L(\theta;D)$ as $T \to \infty$,
	\begin{equation}
		\frac{1}{T} \sum_{t=1}^T L_t \xrightarrow{a.s.} \mathbb{E}(L_t) = L(\theta;D).
	\end{equation}
		
		Thus, $\hat{L} \xrightarrow{a.s.} L(\theta;D)$
	\end{proof}
	
\subsubsection{Decomposition of the hybrid performance measure}

The compounded performance measure $\hat{L}$ can be viewed as a weighted sum of individual performance measures,

\[
\hat{L}_t = \frac{1}{T} \sum_{t=1}^{T} L_t
\]

By linearity of expectation,

\[
E(\hat{L}_t) = E\left(\frac{1}{T} \sum_{t=1}^{T} L_t\right) 
= \frac{1}{T} \sum_{t=1}^{T} E(L_t) 
= L(\theta; D)
\]

Also, by linearity of variance for independent $L_t$,

\[
Var(\hat{L}) = \frac{1}{T^2} \sum_{t=1}^{T} Var(L_t) 
= \frac{1}{T} Var(L_t)
\]

The compounded performance measure $\hat{L}$ can be viewed as a weighted sum of individual performance measures,
	\begin{equation}
		\hat{L}_t = \frac{1}{T} \sum_{t=1}^T L_t
	\end{equation}
	
	By linearity of expectation,
	\begin{align*}
		\mathbb{E}(\hat{L}_t) &= \mathbb{E}\left( \frac{1}{T} \sum_{t=1}^T L_t \right) \\
		&= \frac{1}{T} \sum_{t=1}^T \mathbb{E}(L_t) \\
		&= L(\theta;D)
	\end{align*}
	
	Also, by linearity of variance for independent $L_t$,
	\begin{align*}
		\text{Var}(\hat{L}) &= \frac{1}{T^2} \sum_{t=1}^T \text{Var}(L_t) \\
		&= \frac{1}{T} \text{Var}(L_t)
	\end{align*}
	
\section{Results and Discussion} \label{sec:3.0}
\subsection{Data} \label{sec:3.1}
Three datasets were created for this study, each varying in size---10,000, 50,000, and 100,000---generated using a normal distribution with a mean of 0 and a variance of 1. The datasets were initialised with a consistent random seed of 42, which was updated across trials of 10, 50, and 100, respectively, to ensure reproducibility and consistency in results. The study evaluated three methods of data partitioning: Simple Random Sampling (SRS), K-Fold Cross-Validation (KFCV), and Fusion Sampling Validation (FSV). FSV is a hybrid method combining SRS and KFCV elements to leverage their strengths while mitigating their weaknesses. For KFCV, 5 folds were used with 10 repetitions to ensure robust performance estimation and generalisation capability, enhanced by a scaling factor. FSV employed a weighted factor, adjusting the influence of each iteration's performance on the final composite metric $L^*$, thereby managing overgeneralisation and balancing the bias-variance trade-off. The data partitioning for evaluation was randomised, varying between 60\% and 90\% to ensure a comprehensive assessment of each method's performance.

\subsection{Empirical results}

The study comprehensively analysed statistical metrics across SRS, KFCV, and FSV for each dataset size and trial configuration. These metrics included mean estimates, variance estimates, mean squared error (MSE), bias, rate of convergence of mean estimates, and rate of variance estimates. Results were plotted to compare the performance of each method visually. The analysis aimed to determine which partitioning method provided the most reliable and accurate estimates. SRS, KFCV, and FSV were compared to see how each handled different dataset sizes and trial numbers. The results would indicate the strengths and weaknesses of each method, providing insights into their suitability for different machine learning applications.

The use of FSV, with its hybrid approach, was particularly focused on balancing the trade-offs inherent in SRS and KFCV. By adjusting the influence of each iteration's performance, FSV aimed to reduce over-generalisation and improve the balance between bias and variance. This study's findings are expected to guide data scientists in choosing the most effective data partitioning method for their specific needs, enhancing the accuracy and reliability of machine learning models.
	
\subsubsection*{Dataset size $N = 10,000$}
Table \ref{tab:1} shows statistical properties of partitioned data at N = 10,000, T = 10, 50, and 100, respectively. In 10 trials, the mean estimates were 0.0039 for SRS, 0.0031 for KF, and 0.0037 for FSV. Notably, FSV showed the lowest variance (0.9502) and mean squared error (MSE) (0.9579), indicating its superior accuracy and reliability. Additionally, the bias for FSV was slightly lower (0.0286) compared to SRS and KF, suggesting that FSV's estimates are closer to the true population parameters. The convergence rates for mean and variance estimates were also competitive for FSV, demonstrating its efficiency in reaching stable estimates quickly.

In the case of 50 trials, similar trends were observed. FSV maintained a lower variance (0.9512) and a better MSE (0.9531) than SRS and KF, reaffirming its robustness and precision. The bias and convergence rates for FSV, SRS, and KF remained consistent across methods, indicating that the performance advantages of FSV are stable even as the number of trials increases. This consistency across increased trials highlights FSV's scalability and reliability. Considering 100 trials, it is further confirmed that FSV has the best performance. FSV exhibited the lowest variance (0.9497) and MSE (0.9508) among the three methods. The mean estimates for FSV were very close to those of SRS and KF, but with the added benefits of reduced variance and MSE, FSV provided more reliable and accurate results.
	
	\begin{table}[H]
		\centering
		\caption{Statistical properties of the partitioned data at $N = 10,000$.}
		\begin{tabular}{lccccccccc}
			\toprule
			Statistical Metrics & \multicolumn{3}{c}{10 Trials} & \multicolumn{3}{c}{50 Trials} & \multicolumn{3}{c}{100 Trials} \\
			\cmidrule(lr){2-4} \cmidrule(lr){5-7} \cmidrule(lr){8-10}
			& Mean & Min & Max & Mean & Min & Max & Mean & Min & Max \\
			\midrule
			Mean est. (SRS) & 0.0039 & -0.0203 & 0.0182 & 0.0001 & -0.0114 & 0.0038 & 0.0011 & -0.0064 & 0.0049 \\
			Mean est. KF & 0.0031 & -0.0229 & 0.0128 & 0.0008 & -0.0036 & 0.0064 & 0.0020 & -0.0003 & 0.0073 \\
			Mean est. FSV & 0.0037 & -0.0192 & 0.0173 & 0.0001 & -0.0109 & 0.0036 & 0.0010 & -0.0060 & 0.0046 \\
			Var. est. SRS & 1.0002 & 0.9819 & 1.0131 & 1.0013 & 0.9650 & 1.0142 & 0.9997 & 0.9715 & 1.0051 \\
			Var est. KF & 1.0021 & 0.9776 & 1.0112 & 1.0016 & 0.9852 & 1.0049 & 0.9998 & 0.9794 & 1.0021 \\
			Var est. FSV & 0.9501 & 0.9328 & 0.9624 & 0.9512 & 0.9167 & 0.9635 & 0.9497 & 0.9229 & 0.9548 \\
			MSE SRS & 1.0084 & 0.9235 & 1.0592 & 1.0032 & 0.9714 & 1.0714 & 1.0008 & 0.9844 & 1.0267 \\
			MSE KF & 1.0071 & 0.9169 & 1.0521 & 1.0019 & 0.9683 & 1.0419 & 0.9995 & 0.9781 & 1.0067 \\
			MSE FSV & 0.9579 & 0.8774 & 1.0063 & 0.9530 & 0.9228 & 1.0178 & 0.9508 & 0.9352 & 0.9754 \\
			Bias SRS & 0.0301 & 0.0101 & 0.1463 & 0.0287 & 0.0133 & 0.1719 & 0.0280 & 0.0135 & 0.1539 \\
			Bias KF & 0.0301 & 0.0101 & 0.1463 & 0.0287 & 0.0133 & 0.1719 & 0.0280 & 0.0135 & 0.1539 \\
			Bias FSV & 0.0286 & 0.0096 & 0.1390 & 0.0272 & 0.0127 & 0.1633 & 0.0266 & 0.0128 & 0.1462 \\
			ROC Mean est. SRS & 0.0094 & 0.0033 & 0.0322 & 0.0093 & 0.0043 & 0.0502 & 0.0091 & 0.0044 & 0.0486 \\
			ROC Mean est. KF & 0.0067 & 0.0020 & 0.0491 & 0.0067 & 0.0026 & 0.0411 & 0.0066 & 0.0031 & 0.0385 \\
			ROC Mean est. FSV & 0.0089 & 0.0031 & 0.0306 & 0.0088 & 0.0041 & 0.0477 & 0.0086 & 0.0042 & 0.0461 \\
			ROC Var est. SRS & 0.0124 & 0.0044 & 0.0750 & 0.0126 & 0.0054 & 0.0604 & 0.0123 & 0.0061 & 0.0674 \\
			ROC Var est. KF & 0.0102 & 0.0029 & 0.0664 & 0.0093 & 0.0041 & 0.0502 & 0.0096 & 0.0045 & 0.0578 \\
			ROC Var est. FSV & 0.0118 & 0.0042 & 0.0714 & 0.0120 & 0.0051 & 0.0575 & 0.0117 & 0.0058 & 0.0640 \\
			\bottomrule
		\end{tabular} \label{tab:1}
	\end{table}
	
Figures \ref{fig:1a} to \ref{fig:1c} are plots of statistical properties of partitioned data at $N = 10,000, T = 10, 50, \mbox{ and } 100$, respectively. Various estimation methods exhibit distinct behaviours across different metrics for a sample size of $N = 10,000$. Regarding mean estimates, SRS initially shows significant variance at T = 10, whereas KF provides more stable estimates, and FSV demonstrates the highest stability. As the number of trials ($T$) increases to 50 and 100, SRS stabilises, KF remains consistent, and FSV shows the least variance, indicating superior stability. Moving to variance estimates, at N = 10,000, SRS exhibits considerable variability initially, which stabilises as $T$ increases. KF offers more stable variance estimates from the outset, while FSV consistently demonstrates the least variance, highlighting its superior stability across different $T$ values.

	\begin{figure}[H]
		\centering
		\includegraphics[width=1.0\textwidth]{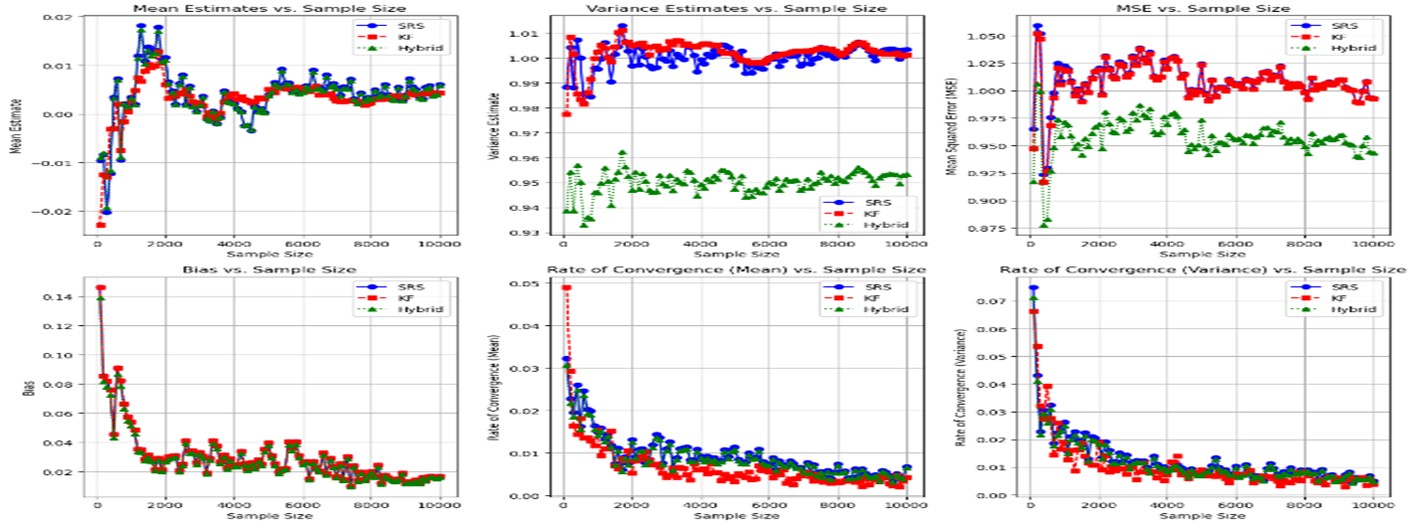}
		\caption{Plots of statistical properties of partitioned data at $N = 10,000, T = 10$.}
		\label{fig:1a}
	\end{figure}
	
	\begin{figure}[H]
		\centering
		\includegraphics[width=1.0\textwidth]{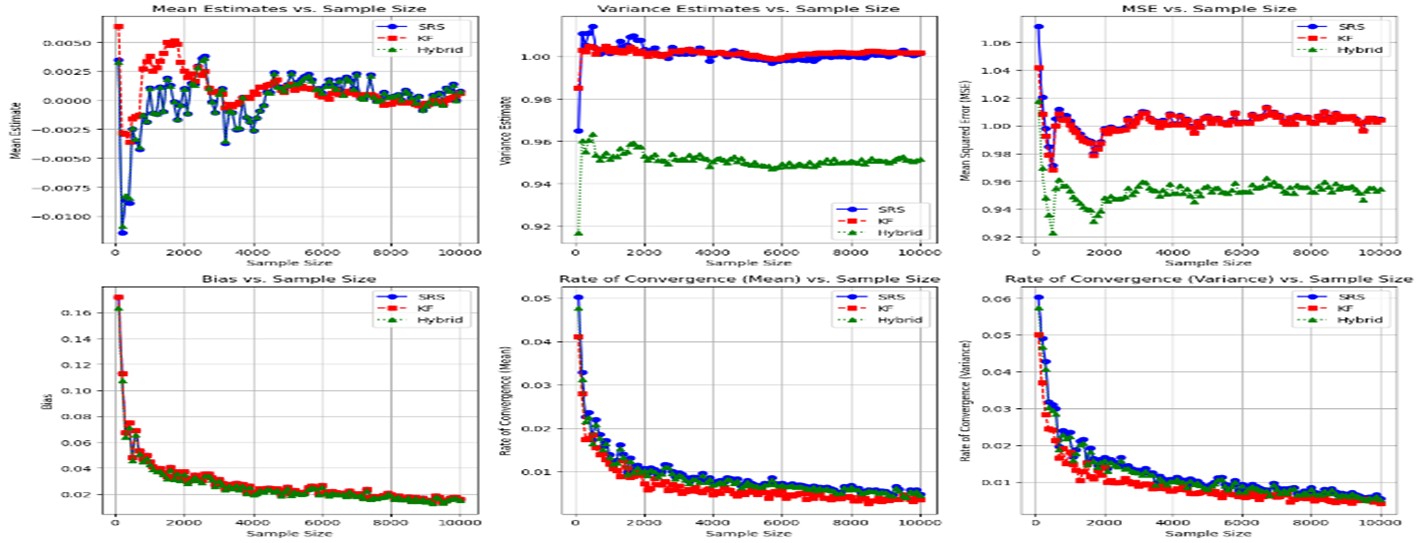}
		\caption{Plots of statistical properties of partitioned data at $N = 10,000, T = 50$.}
		\label{fig:1b}
	\end{figure}
	
	\begin{figure}[H]
		\centering
		\includegraphics[width=1.0\textwidth]{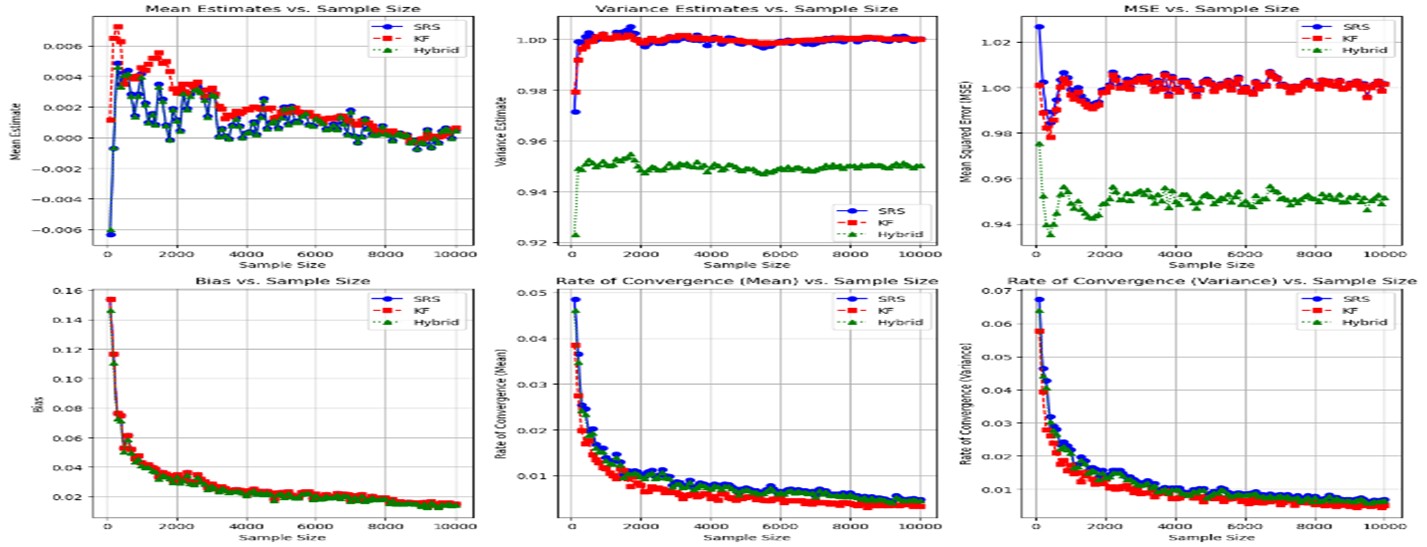}
		\caption{Plots of statistical properties of partitioned data at $N = 10,000, T = 100$.}
		\label{fig:1c}
	\end{figure}
	
Mean Squared Error (MSE) comparisons show SRS starting with a higher MSE due to initial variance, which decreases with increasing T. In contrast, KF maintains a low MSE consistently, and FSV consistently presents the lowest MSE, emphasising its accuracy. Regarding bias, SRS initially exhibits higher bias, which diminishes with increasing T, while KF maintains consistently low bias, and FSV consistently shows the least bias, underscoring its reliability. Regarding the convergence rate for mean estimates, SRS starts with slower convergence due to initial variance, which improves with higher T. In comparison, KF achieves faster convergence through fold averaging, and FSV consistently demonstrates the fastest convergence rate. Similarly, SRS converges more slowly in variance estimates due to higher inherent variance, improving with higher T. At the same time, KF maintains a rapid pace, and FSV consistently shows the fastest convergence, confirming its efficiency.
	
\subsubsection*{Dataset size $N = 50,000$}
The statistical properties of partitioned data with a size of 50,000 across 10, 50, and 100 trials are shown in Table \ref{tab:2}. For 10 trials, FSV showed superior performance with an average mean estimate of 0.0012, variance of 0.9487, and MSE of 0.9535, indicating the lowest variance and best MSE. FSV's bias was slightly lower at 0.0131 compared to 0.0137 for SRS and KF. The convergence rates for FSV were 0.0041 for mean estimates and 0.0056 for variance estimates. In comparison, SRS had a mean estimate of 0.0013, variance of 0.9987, and MSE of 1.0037, while KF had a mean estimate of 0.0007, variance of 1.0000, and MSE of 1.0033.

In the 50 trials, FSV's mean estimates averaged 0.0001, variance was 0.9501, and MSE was 0.9521. FSV's bias was 0.0123, better than the 0.0129 for SRS and KF. The convergence rates for FSV were 0.0039 for mean estimates and 0.0057 for variance estimates. SRS had a variance of 1.0001 and MSE of 1.0022, while KF had a variance of 1.0006 and MSE of 1.0019. In 100 trials, FSV maintained its superior performance with mean estimates at 0.0001, variance at 0.9493, and MSE at 0.9501. FSV's bias was 0.0122, compared to 0.0129 for SRS and KF. The convergence rates for FSV were 0.0039 for mean estimates and 0.0056 for variance estimates. SRS had a variance of 0.9993 and an MSE of 1.0001, while KF had a variance of 0.9995 and an MSE of 0.9998.
	
	\begin{table}[H]
		\centering
		\caption{Statistical properties of partitioned data at $N = 50,000$.}
		\begin{tabular}{lccccccccc}
			\toprule
			Statistical Metrics & \multicolumn{3}{c}{10 Trials} & \multicolumn{3}{c}{50 Trials} & \multicolumn{3}{c}{100 Trials} \\
			\cmidrule(lr){2-4} \cmidrule(lr){5-7} \cmidrule(lr){8-10}
			& Mean & Min & Max & Mean & Min & Max & Mean & Min & Max \\
			\midrule
			Mean est. (SRS) & 0.0013 & -0.0203 & 0.0182 & 0.0001 & -0.0114 & 0.0038 & 0.0001 & -0.0064 & 0.0049 \\
			Mean est. KF & 0.0007 & -0.0229 & 0.0128 & 0.0001 & -0.0036 & 0.0064 & 0.0002 & -0.0009 & 0.0073 \\
			Mean est. FSV & 0.0012 & -0.0192 & 0.0173 & 0.0001 & -0.0109 & 0.0036 & 0.0001 & -0.0060 & 0.0046 \\
			Var. est. SRS & 0.9987 & 0.9819 & 1.0131 & 1.0001 & 0.9650 & 1.0142 & 0.9993 & 0.9715 & 1.0051 \\
			Var est. KF & 1.0000 & 0.9776 & 1.0112 & 1.0006 & 0.9852 & 1.0049 & 0.9995 & 0.9794 & 1.0021 \\
			Var est. FSV & 0.9487 & 0.9328 & 0.9624 & 0.9501 & 0.9167 & 0.9635 & 0.9493 & 0.9229 & 0.9548 \\
			MSE SRS & 1.0037 & 0.9235 & 1.0592 & 1.0022 & 0.9714 & 1.0714 & 1.0001 & 0.9844 & 1.0267 \\
			MSE KF & 1.0033 & 0.9169 & 1.0521 & 1.0019 & 0.9683 & 1.0419 & 0.9998 & 0.9781 & 1.0067 \\
			MSE FSV & 0.9535 & 0.8774 & 1.0063 & 0.9521 & 0.9228 & 1.0178 & 0.9501 & 0.9352 & 0.9754 \\
			Bias SRS & 0.0137 & 0.0041 & 0.1463 & 0.0129 & 0.0055 & 0.1719 & 0.0129 & 0.0060 & 0.1539 \\
			Bias KF & 0.0137 & 0.0041 & 0.1463 & 0.0129 & 0.0055 & 0.1719 & 0.0129 & 0.0060 & 0.1539 \\
			Bias FSV & 0.0130 & 0.0039 & 0.1390 & 0.0123 & 0.0052 & 0.1633 & 0.0122 & 0.0057 & 0.1462 \\
			ROC Mean est. SRS & 0.0043 & 0.0013 & 0.0322 & 0.0041 & 0.0017 & 0.0502 & 0.0041 & 0.0019 & 0.0486 \\
			ROC Mean est. KF & 0.0031 & 0.0007 & 0.0491 & 0.0031 & 0.0013 & 0.0411 & 0.0031 & 0.0014 & 0.0385 \\
			ROC Mean est. FSV & 0.0041 & 0.0013 & 0.0306 & 0.0039 & 0.0016 & 0.0477 & 0.0039 & 0.0018 & 0.0461 \\
			ROC Var est. SRS & 0.0059 & 0.0015 & 0.0750 & 0.0060 & 0.0025 & 0.0604 & 0.0059 & 0.0029 & 0.0674 \\
			ROC Var est. KF & 0.0045 & 0.0010 & 0.0664 & 0.0043 & 0.0018 & 0.0502 & 0.0044 & 0.0020 & 0.0578 \\
			ROC Var est. FSV & 0.0056 & 0.0014 & 0.0714 & 0.0057 & 0.0024 & 0.0575 & 0.0056 & 0.0027 & 0.0640 \\
			\bottomrule
		\end{tabular} \label{tab:2}
	\end{table}
	
The plots of the statistical properties of partitioned data with a size of 50,000 evaluated at 10, 50, and 100 trials are presented in Figures \ref{fig:2a} to \ref{fig:2c}. For a larger sample size of $N = 50,000$, the trends in estimation methods evolve. Mean estimates from SRS become more stable even at $T = 10$, benefiting from the larger sample size. At the same time, KF remains reliable, and FSV outperforms stability and accuracy across higher T values. Variance estimates show SRS stabilising at $T = 10$, with KF maintaining steadiness and FSV demonstrating superior stability consistently.

MSE analysis reveals SRS benefiting from reduced variance, KF maintaining steady low MSE, and FSV consistently providing the most accurate MSE estimates, especially at higher T values. Bias estimates improve across all methods, with SRS showing reduced bias, KF maintaining low bias consistently, and FSV continuing to exhibit the least bias. The convergence rate for mean estimates shows SRS converging faster due to the larger sample size, KF maintaining rapid convergence, and FSV consistently achieving the fastest convergence, emphasising its effectiveness. Similarly, in variance estimates, SRS benefits from reduced variance and converges faster, KF maintains swift convergence, and FSV continues to lead with the fastest convergence rate.
	
	\begin{figure}[H]
		\centering
		\includegraphics[width=1.0\textwidth]{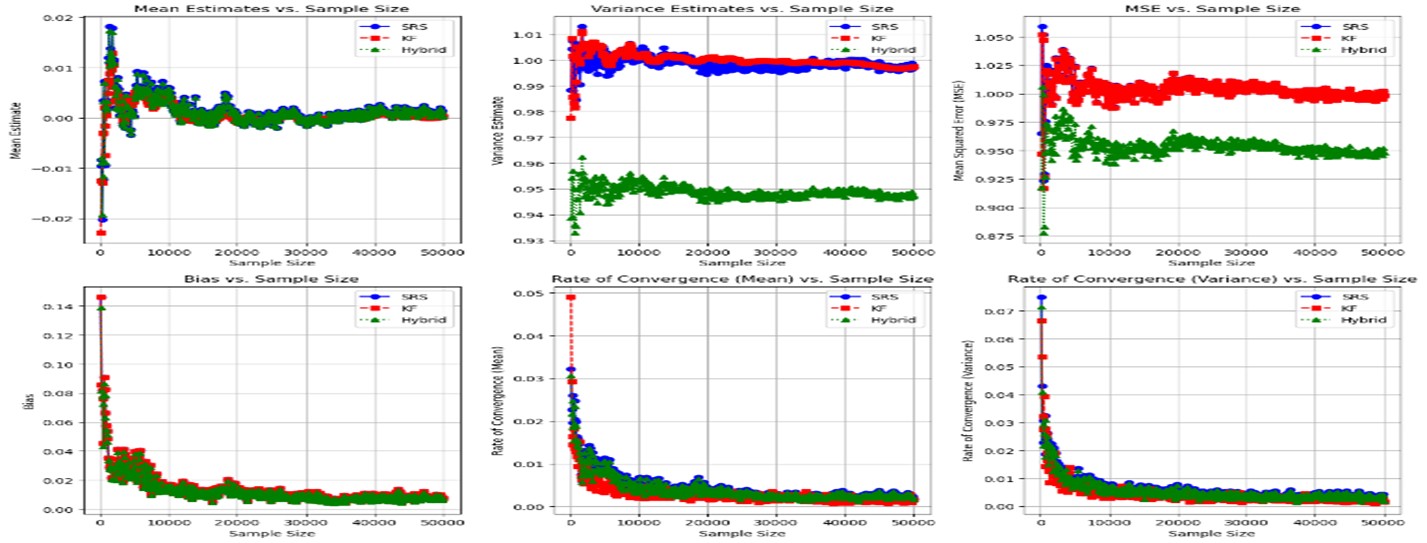}
		\caption{Plots of statistical properties of partitioned data at N$ = 50,000, T = 10$.}
		\label{fig:2a}
	\end{figure}
	
	\begin{figure}[H]
		\centering
		\includegraphics[width=1.0\textwidth]{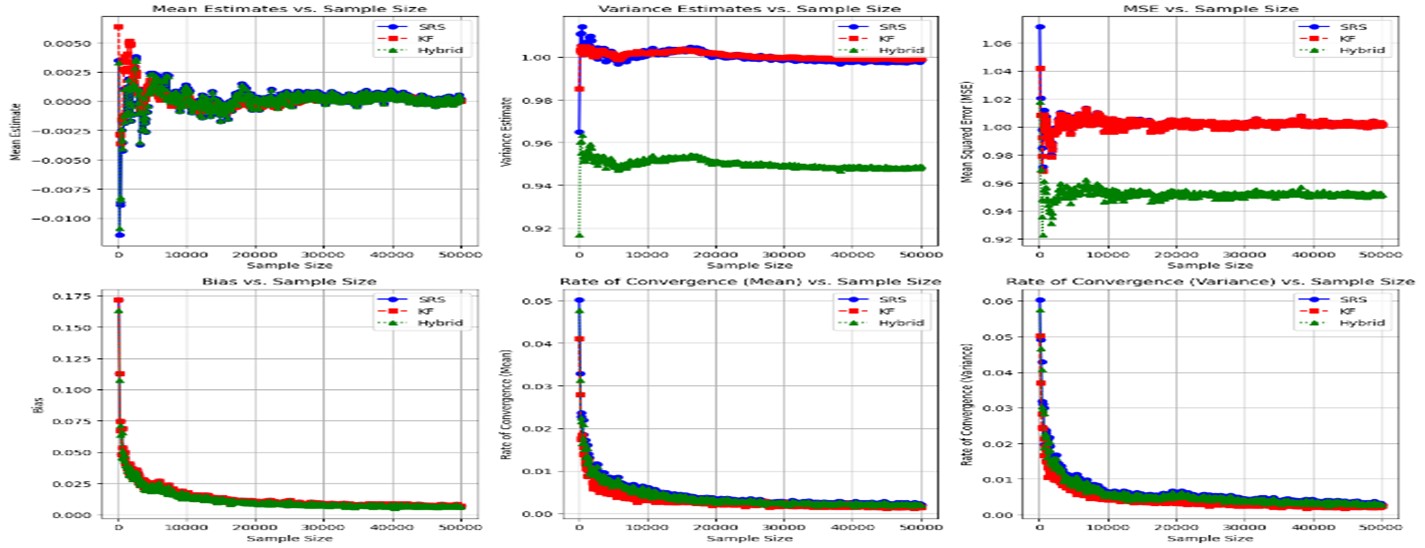}
		\caption{Plots of statistical properties of partitioned data at $N = 50,000, T = 50$.}
		\label{fig:2b}
	\end{figure}
	
	\begin{figure}[H]
		\centering
		\includegraphics[width=1.0\textwidth]{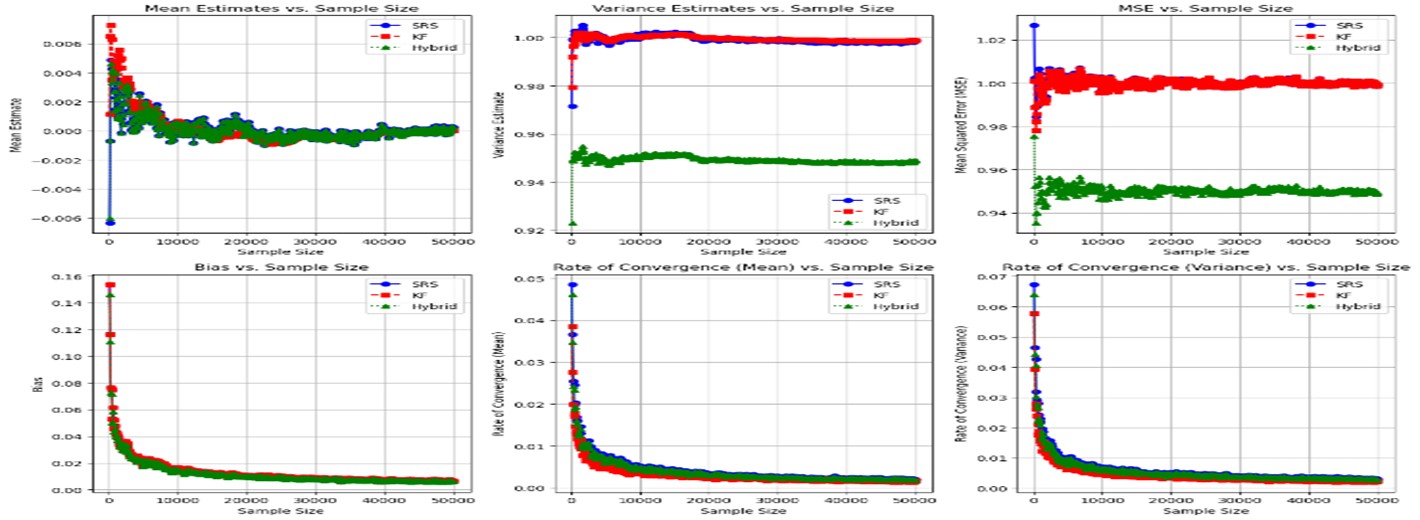}
		\caption{Plots of statistical properties of partitioned data at $N = 100,000, T = 100$.}
		\label{fig:2c}
	\end{figure}

\subsubsection*{Dataset Size $N = 100,000$}
The statistical properties of partitioned data with a size of 100,000 were evaluated across 10, 50, and 100 trials, as shown in Table \ref{tab:3}. For 10 trials, FSV demonstrated superior performance with an average mean estimate of 0.0015, variance of 0.9490, and MSE of 0.9504, indicating the lowest variance and best MSE. FSV's bias was slightly lower at 0.0092 compared to 0.0097 for both SRS and KF. The convergence rates for FSV were 0.0029 for mean estimates and 0.0039 for variance estimates. In contrast, SRS had a mean estimate of 0.0016, variance of 0.9989, and MSE of 1.0004, while KF had a mean estimate of 0.0010, variance of 0.9993, and MSE of 1.0002.

In the 50 trials, FSV continued to excel with mean estimates averaging 2.77e-05, variance at 0.9496, and MSE at 0.9515. FSV's bias was 0.0087, better than 0.0092 for both SRS and KF. The convergence rates for FSV were 0.0028 for mean estimates and 0.0040 for variance estimates. SRS had a variance of 0.9996 and an MSE of 1.0016, while KF had a variance of 1.0000 and an MSE of 1.0014. Furthermore, for 100 trials, FSV maintained superior performance with mean estimates at 5.09e-05, variance at 0.9491, and MSE at 0.9498. FSV's bias was 0.0087, compared to 0.0091 for SRS and KF. The convergence rates for FSV were 0.0028 for mean estimates and 0.0040 for variance estimates. SRS had a variance of 0.9990 and an MSE of 0.9998, while KF had a variance of 0.9992 and an MSE of 0.9996.

	\begin{table}[H]
		\centering
		\caption{Statistical properties of partitioned data at $N = 100,000$.}
		\begin{tabular}{lccccccccc}
			\toprule
			Statistical Metrics & \multicolumn{3}{c}{10 Trials} & \multicolumn{3}{c}{50 Trials} & \multicolumn{3}{c}{100 Trials} \\
			\cmidrule(lr){2-4} \cmidrule(lr){5-7} \cmidrule(lr){8-10}
			& Mean & Min & Max & Mean & Min & Max & Mean & Min & Max \\
			\midrule
			Mean est. (SRS) & 0.0016 & -0.0203 & 0.0182 & 0.0000 & -0.0114 & 0.0038 & 0.0001 & -0.0064 & 0.0049 \\
			Mean est. KF & 0.0010 & -0.0229 & 0.0128 & 0.0000 & -0.0036 & 0.0064 & 0.0001 & -0.0009 & 0.0073 \\
			Mean est. FSV & 0.0015 & -0.0192 & 0.0173 & 0.0000 & -0.0109 & 0.0036 & 0.0001 & -0.0060 & 0.0046 \\
			Var. est. SRS & 0.9989 & 0.9819 & 1.0131 & 0.9996 & 0.9650 & 1.0142 & 0.9990 & 0.9715 & 1.0051 \\
			Var est. KF & 0.9993 & 0.9776 & 1.0112 & 1.0000 & 0.9852 & 1.0049 & 0.9992 & 0.9794 & 1.0021 \\
			Var est. FSV & 0.9490 & 0.9328 & 0.9624 & 0.9496 & 0.9167 & 0.9635 & 0.9490 & 0.9229 & 0.9548 \\
			MSE SRS & 1.0004 & 0.9235 & 1.0592 & 1.0016 & 0.9714 & 1.0714 & 0.9998 & 0.9844 & 1.0267 \\
			MSE KF & 1.0002 & 0.9169 & 1.0521 & 1.0014 & 0.9683 & 1.0419 & 0.9996 & 0.9781 & 1.0067 \\
			MSE FSV & 0.9504 & 0.8774 & 1.0063 & 0.9515 & 0.9228 & 1.0178 & 0.9498 & 0.9352 & 0.9754 \\
			Bias SRS & 0.0097 & 0.0014 & 0.1463 & 0.0092 & 0.0034 & 0.1719 & 0.0091 & 0.0037 & 0.1539 \\
			Bias KF & 0.0097 & 0.0014 & 0.1463 & 0.0092 & 0.0034 & 0.1719 & 0.0091 & 0.0037 & 0.1539 \\
			Bias FSV & 0.0092 & 0.0013 & 0.1390 & 0.0087 & 0.0032 & 0.1633 & 0.0087 & 0.0035 & 0.1462 \\
			ROC Mean est. SRS & 0.0031 & 0.0007 & 0.0322 & 0.0029 & 0.0011 & 0.0502 & 0.0029 & 0.0013 & 0.0486 \\
			ROC Mean est. KF & 0.0022 & 0.0005 & 0.0491 & 0.0022 & 0.0009 & 0.0411 & 0.0022 & 0.0010 & 0.0385 \\
			ROC Mean est. FSV & 0.0029 & 0.0006 & 0.0306 & 0.0028 & 0.0011 & 0.0477 & 0.0028 & 0.0012 & 0.0461 \\
			ROC Var est. SRS & 0.0041 & 0.0010 & 0.0750 & 0.0042 & 0.0017 & 0.0604 & 0.0042 & 0.0018 & 0.0674 \\
			ROC Var est. KF & 0.0032 & 0.0007 & 0.0664 & 0.0031 & 0.0010 & 0.0502 & 0.0031 & 0.0013 & 0.0578 \\
			ROC Var est. FSV & 0.0039 & 0.0010 & 0.0714 & 0.0040 & 0.0016 & 0.0575 & 0.0040 & 0.0017 & 0.0640 \\
			\bottomrule
		\end{tabular} \label{tab:3}
	\end{table}
	
Figures \ref{fig:3a} to \ref{fig:3c} are plots of statistical properties of partitioned data at N = 100,000, T = 10, 50, and 100, respectively. At the largest sample size of N = 100,000, estimation methods demonstrate heightened stability and accuracy across all metrics. SRS exhibits highly stable mean estimates with minimal variance, KF provides consistent estimates, and FSV remains the best performer with the least variance, confirming its effectiveness. Variance estimates from SRS are highly stable, KF offers reliable results, and FSV consistently demonstrates the least variance, underscoring its superiority.

MSE analysis shows SRS achieving lower MSE, KF maintaining a consistently low MSE, and FSV consistently providing the lowest MSE, highlighting its accuracy and reliability. Bias estimates further improve, with SRS showing minimal bias, KF maintaining low bias consistently, and FSV continuing to present the least bias across varying T values. The convergence rate for both mean and variance estimates reflects accelerated performance across all methods, with SRS converging significantly faster, KF maintaining a rapid pace, and FSV consistently achieving the fastest convergence, emphasising its efficiency and reliability across diverse scenarios.

\begin{figure}[H]
	\centering
	\includegraphics[width=1.0\textwidth]{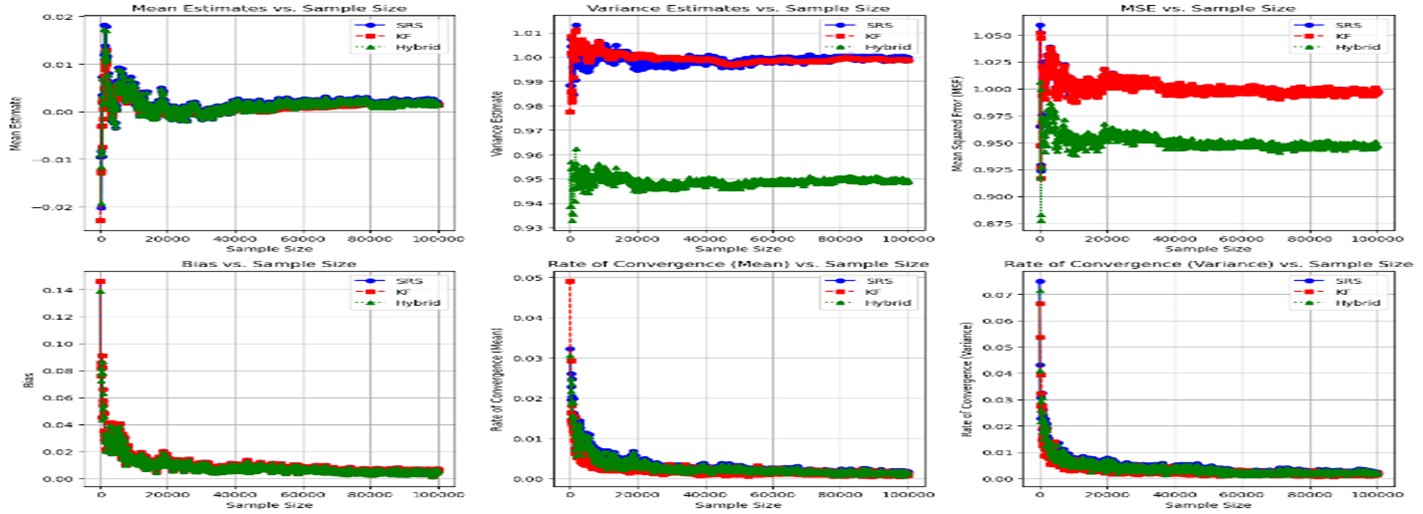}
	\caption{Plots of statistical properties of partitioned data at $N = 100,000, T = 10$.}
	\label{fig:3a}
\end{figure}

\begin{figure}[H]
	\centering
	\includegraphics[width=1.0\textwidth]{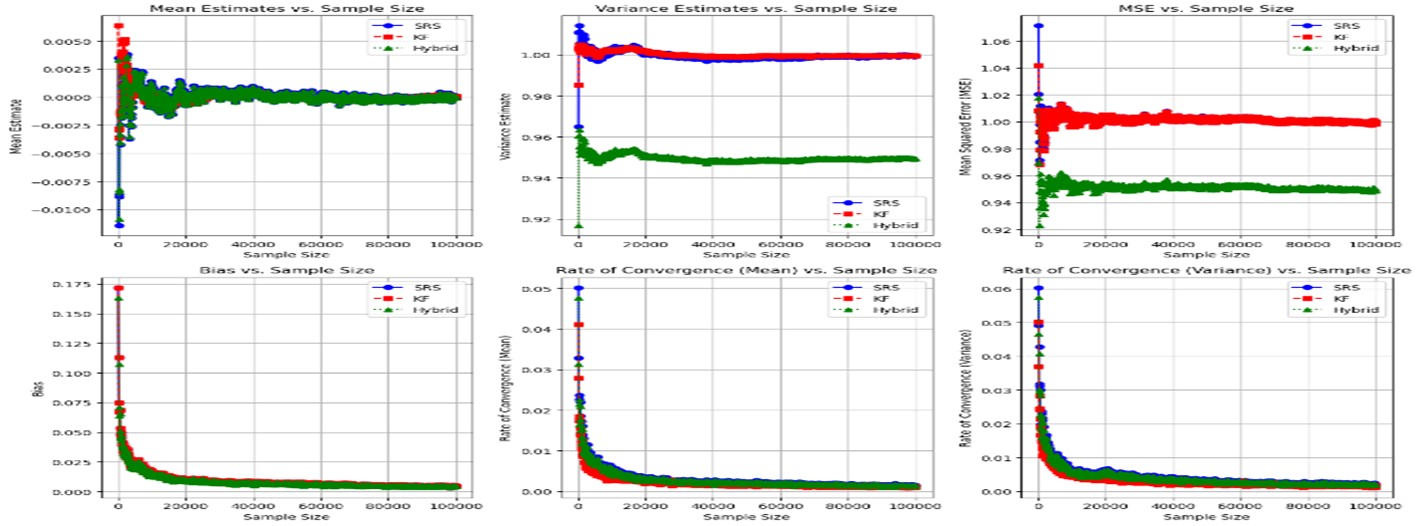}
	\caption{Plots of Statistical properties of partitioned data at $N = 100,000, T = 50$.}
	\label{fig:3b}
\end{figure}

\begin{figure}[H]
	\centering
	\includegraphics[width=1.0\textwidth]{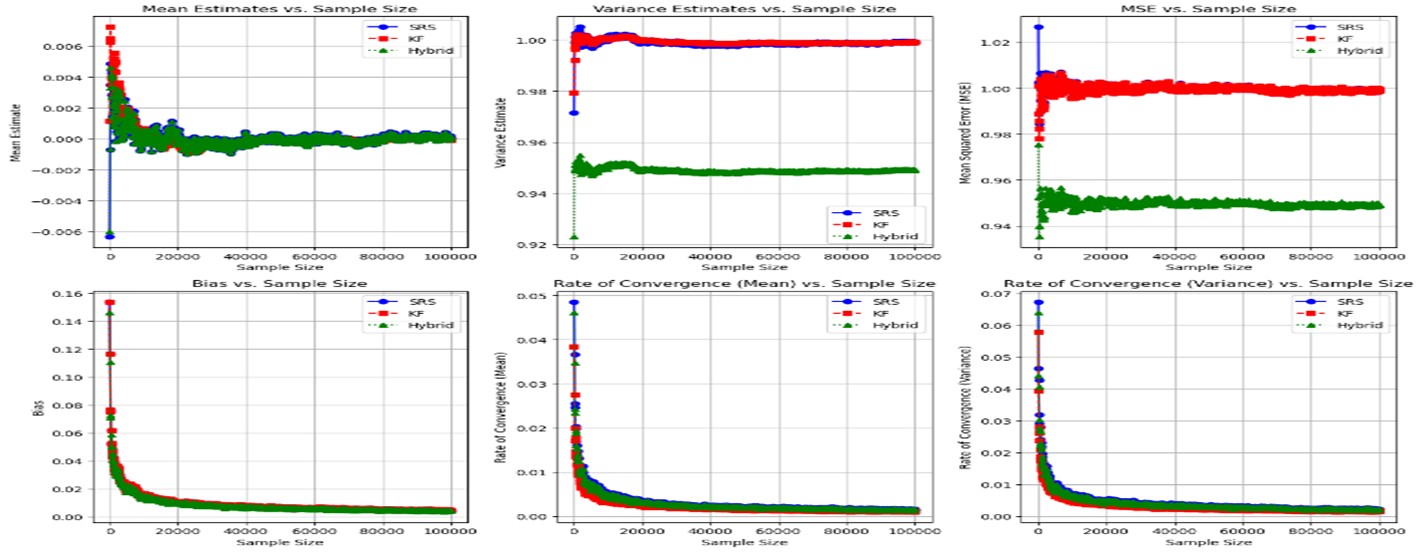}
	\caption{Plots of statistical properties of partitioned data at $N = 100,000, T = 100$.}
	\label{fig:3c}
\end{figure}

\section{Discussion} \label{sec:4.0}
Data partitioning techniques like Simple Random Sampling (SRS) and K-Fold Cross-Validation (KFCV) are fundamental for model training and evaluation in machine learning. SRS ensures that training and validation samples are representative, which enhances model generalisation. However, SRS can lead to non-representative training sets, particularly with imbalanced data distributions, and exhibits high variance and bias with smaller sample sizes, stabilising only with more trials and larger sample sizes. On the other hand, KFCV provides robust, stable, and consistent estimates, making it a reliable method from the outset. Despite these strengths, KFCV demands significant computational resources for large datasets and requires meticulous data shuffling, which can be computationally intensive and time-consuming.

To address these challenges, a hybrid approach, the Fusion Sampling Validation (FSV), is a better method that combines the strengths of both SRS and KFCV. FSV mitigates the computational demands and non-representative sample issues faced by SRS and KFCV, optimising data partitioning and enhancing the efficiency and performance of machine learning models. FSV consistently demonstrates the highest stability, least variance, lowest MSE, and least bias across all sample sizes and trial numbers. Additionally, it achieves the fastest convergence rates for both mean and variance estimates, emphasising its effectiveness and reliability.

The implications of these findings are significant. FSV stands out as the most effective estimation method across all metrics for sample sizes of $N = 10,000, N = 50,000, \mbox{ and } N = 100,000$, surpassing the performance of SRS and KFCV. This hybrid method optimises data partitioning by balancing computational efficiency and representativeness, ensuring better model generalisation and reliability. For smaller datasets, KFCV may still be preferred due to its robustness. However, FSV emerges as the best choice for larger datasets, leveraging the strengths and mitigating the weaknesses of both SRS and KFCV.

Some hybrid validation techniques, such as Monte Carlo CV and Stratified K-Fold CV, among others, are suggested in the literature to overcome the drawbacks of conventional data splitting. These techniques integrate random sampling with replacement or stratification (for imbalanced datasets) with cross-validation to enhance representativeness, stability, and computational efficiency. Nevertheless, a comprehensive comparison between them and the novel Fusion Sampling Validation (FSV) is still missing. Future research could empirically contrast FSV with these hybrids, particularly for large-scale, imbalanced, or high-dimensional data.

\section{Conclusion} \label{sec:5.0}
The development and adoption of the FSV are crucial for overcoming the limitations of traditional data partitioning techniques. By enhancing machine learning models' stability, accuracy, and efficiency, FSV represents a significant advancement in data partitioning strategies. Continued research and development of hybrid approaches will further address the evolving challenges in machine learning and data science, ensuring optimised model performance and reliability across diverse applications.

	\section*{Acknowledgment}
	Authors acknowledge the Associateship Program, QLS Section, Abdus Salam International Centre for Theoretical Physics, Trieste, Italy.
	
	\section*{Funding}
	This research was funded by the Associates Programme of the Abdus Salam International Centre for Theoretical Physics, Trieste, Italy.
	
\section*{Author contributions}

\textbf{Christopher Godwin Udomboso:} Conceptualization, Methodology, Software, Funding acquisition  \textbf{Caston Sigauke:} Writing- Reviewing and Editing, Validation, Project administration  \textbf{Ini Adinya:} Writing- Original draft preparation,  Investigation, Software Validation. All authors have read and agreed to the published version of the manuscript.


\section*{Conflicts of Interest:} The corresponding author declares that there are no conflicts of interest on behalf of all authors.

\section*{Abbreviations}
{The following abbreviations are used in this manuscript:\\
	\noindent
	\begin{tabular}{@{}ll}
	DNN & Deep Neural Network \\
	FSV & Fusion Sampling Validation \\
	ME & Mean Error \\
	MSE & Mean Squared Error \\
	KFCV & K-Fold Cross Validation \\
	SRS & Simple Random Sampling \\		
	VE & Variance estimate \\
	\end{tabular}
}



\begin{thebibliography}{9}
	
\bibitem{bengio2004}	Bengio, Y., \& Grandvalet, Y. (2004). No unbiased estimator of the variance of K-fold cross-validation. Journal of Machine Learning Research, 5, 1089--1105.

\bibitem{bergstra2012}	Bergstra, J., \& Bengio, Y. (2012). Random search for hyper-parameter optimization. Journal of Machine Learning Research, 13(Feb), 281--305.

\bibitem{bishop1995}	Bishop, C. M. (1995). Neural networks for pattern recognition. Oxford University Press.

\bibitem{bishop2006}	Bishop, C. M. (2006). Pattern recognition and machine learning. Springer.

\bibitem{breiman1996}	Breiman, L. (1996). Bagging predictors. Machine Learning, 24(2), 123--140.

\bibitem{chen2017}	Chen, L.-C., Papandreou, G., Kokkinos, I., Murphy, K., \& Yuille, A. L. (2017). Deeplab: Semantic image segmentation with deep convolutional nets, atrous convolution, and fully connected CRFs. IEEE Transactions on Pattern Analysis and Machine Intelligence, 40(4), 834--848. \url{https://doi.org/10.1109/TPAMI.2017.2699184}

\bibitem{chen2014}	Chen, X., \& Xie, M. (2014). A split-and-conquer approach for analysis of extraordinarily large data. Statistica Sinica, 24(4), 1655--1684.

\bibitem{courbariaux2015}	Courbariaux, M., Bengio, Y., \& David, J.-P. (2015). Binaryconnect: Training deep neural networks with binary weights during propagations. In Advances in Neural Information Processing Systems (NeurIPS), 3123--3131.

\bibitem{cybenko1989}	Cybenko, G. (1989). Approximation by superpositions of a sigmoidal function. Mathematics of Control, Signals, and Systems, 2(4), 303--314.

\bibitem{downing2023}	Downing, M., \& Bultan, T. K. (2023). The case for scalable quantitative neural network analysis. In Learned Components (SE4SafeML '23), San Francisco, CA, USA. ACM, New York, NY, USA. \url{https://doi.org/10.1145/3617574.3617862}

\bibitem{freund1997}	Freund, Y., \& Schapire, R. E. (1997). A decision-theoretic generalization of on-line learning and an application to boosting. Journal of Computer and System Sciences, 55(1), 119--139.

\bibitem{genuer2017}	Genuer, R., Poggi, J.-M., Tuleau-Malot, C., \& Villa-Vialaneix, N. (2017). Random forests for big data. Big Data Research, 9, 28--46.

\bibitem{han2016}	Han, S., Mao, H., \& Dally, W. J. (2016). Deep compression: Compressing deep neural networks with pruning, trained quantization and Huffman coding. In International Conference on Learning Representations (ICLR).

\bibitem{hastie2005}	Hastie, T., Tibshirani, R., Friedman, J., \& Franklin, J. (2005). The elements of statistical learning: Data mining, inference and prediction. The Mathematical Intelligencer, 27(2), 83--85. \url{https://doi.org/10.1007/BF02985802}

\bibitem{haynes2006}	Haynes, J. D., \& Rees, G. (2006). Decoding mental states from brain activity in humans. Nature Reviews Neuroscience, 7(7), 523--534. \url{https://doi.org/10.1038/nrn1931}

\bibitem{hornik1989}	Hornik, K., Stinchcombe, M., \& White, H. (1989). Multilayer feedforward networks are universal approximators. Neural Networks, 2(5), 359--366.

\bibitem{howard2017}	Howard, A. G., Zhu, M., Chen, B., Kalenichenko, D., Wang, W., Weyand, T., Adam, H. (2017). MobileNets: Efficient convolutional neural networks for mobile vision applications. In Proceedings of the IEEE Conference on Computer Vision and Pattern Recognition (CVPR).

\bibitem{huang2017}	Huang, G., Liu, Z., Maaten, L., \& Weinberger, K. Q. (2017). Densely connected convolutional networks. In Proceedings of the IEEE Conference on Computer Vision and Pattern Recognition (CVPR), 4700--4708.

\bibitem{kohavi1995}	Kohavi, R. (1995). A study of cross-validation and bootstrap for accuracy estimation and model selection. In International Joint Conference on Artificial Intelligence (pp. 1137--1143).

\bibitem{korjus2016}	Korjus, K., Hebart, M. N., \& Vicente, R. (2016). An efficient data partitioning to improve classification performance while keeping parameters interpretable. PLoS ONE, 11(8), e0161788. \url{https://doi.org/10.1371/journal.pone.0161788}

\bibitem{krizhesky2017}	Krizhevsky, A., Sutskever, I., \& Hinton, G. E. (2017). ImageNet classification with deep convolutional neural networks. Communications of the ACM, 60(6), 84--90.

\bibitem{larranaga2006}	Larrañaga, P., Calvo, B., Santana, R., Bielza, C., Galdiano, J., Inza, I., Robles, V. (2006). Machine learning in bioinformatics. Briefings in Bioinformatics, 7(1), 86--112. \url{https://doi.org/10.1093/bib/bbk007}

\bibitem{lazar2018}	Lazar, N. (2018). The big picture: Divide and combine to conquer big data. Chance, 31(1), 57--59.

\bibitem{lecun1990}	LeCun, Y., Boser, B., Denker, J., Henderson, D., Howard, R., Hubbard, W., \& Jackel, L. (1990). Handwritten digit recognition with a backpropagation network. In D. S. Touretzky (Ed.), NIPS2* (pp. 396--404). Morgan Kaufmann.

\bibitem{lorraine2020}	Lorraine, J., Vicol, P., \& Duvenaud, D. (2020). Optimizing millions of hyperparameters by implicit differentiation. In Proceedings of the 23rd International Conference on Artificial Intelligence and Statistics (AISTATS). Retrieved from arXiv:1911.02590

\bibitem{mcculloch1943}	McCulloch, W. S., \& Pitts, W. (1943). A logical calculus of the ideas immanent in nervous activity. The Bulletin of Mathematical Biophysics, 5(4), 115--133.
	
\bibitem{mlodozeniec2023}	Mlodozeniec, B., Reisser, M., \& Louizos, C. (2023). Hyperparameter optimization through neural network partitioning. In ICLR 2023.

\bibitem{pereira2009}	Pereira, F., Mitchell, T., \& Botvinick, M. (2009). Machine learning classifiers and fMRI: A tutorial overview. NeuroImage, 45(1), S199--S209. \url{https://doi.org/10.1016/j.neuroimage.2008.11.007}
	
\bibitem{redmon2016}	Redmon, J., Divvala, S., Girshick, R., \& Farhadi, A. (2016). You only look once: Unified, real-time object detection. In Proceedings of the IEEE Conference on Computer Vision and Pattern Recognition (CVPR), 779--788. \url{https://doi.org/10.1109/CVPR.2016.91}
	
\bibitem{reed1993}	Reed, R. (1993). Pruning algorithms: A survey. IEEE Transactions on Neural Networks, 4, 740--747. \url{https://doi.org/10.1109/TNN.1993.683611}
	
\bibitem{skouras1994}	Skouras, K., Goutis, C., \& Bramson, M. J. (1994). Estimation in linear models using gradient descent with early stopping. Statistics and Computing, 4, 271–278. \url{https://doi.org/10.1007/BF00141417}
	
\bibitem{tetko1997}	Tetko, I. V., \& Villa, A. E. P. (1997). Efficient partition of learning data sets for neural network training. Neural Networks, 10(8), 1361--1374. \url{https://doi.org/10.1016/S0893-6080(97)00056-3}
	
\bibitem{tetko1995}	Tetko, I. V., Livingstone, D. J., \& Luik, A. I. (1995). Neural network studies. 1: Comparison of overfitting and overtraining. Journal of Chemical Information and Computer Sciences, 35, 826--833. \url{https://doi.org/10.1021/ci00027a002}
	
\bibitem{tibshirani1996}	Tibshirani, R. (1996). Regression shrinkage and selection via the lasso. Journal of the Royal Statistical Society: Series B (Methodological), 58(1), 267--288.
	
\bibitem{varma2006}	Varma, S., \& Simon, R. (2006). Bias in error estimation when using cross-validation for model selection. BMC Bioinformatics, 7(1), 91. \url{https://doi.org/10.1186/1471-2105-7-91}
	
\bibitem{wu2024}	Wu, K., \& Politis, D. M. (2024). Scalable subsampling inference for deep neural networks. arXiv preprint arXiv:2405.08276.

\bibitem{yu2019}	Yu, J., Yang, L., Xu, N., Yang, J., \& Huang, T. (2019). Slimmable neural networks. In International Conference on Learning Representations (ICLR).
	
\bibitem{zhang2019}	Zhang, L., Tan, Z., Song, J., Chen, J., Bao, C., \& Ma, K. (2019). SCAN: A scalable neural networks framework towards compact and efficient models. In Advances in Neural Information Processing Systems (NeurIPS), 33rd Conference, Vancouver, Canada.
	
\end{thebibliography}

\end{document}